\documentclass[conference]{IEEEtran}

\usepackage{amsmath}
\usepackage{amsfonts}
\usepackage{amssymb}
\usepackage{mathrsfs}     
\usepackage[justification=raggedright]{subcaption}
\usepackage[usenames,dvipsnames]{color}
\usepackage{mathtools}

\newtheorem{lemma}{Lemma}
\newtheorem{theorem}{Theorem}

\DeclarePairedDelimiter\ceil{\lceil}{\rceil}
\DeclarePairedDelimiter\floor{\lfloor}{\rfloor}

\usepackage{times}

\usepackage[numbers]{natbib}
\usepackage{multicol}
\usepackage[bookmarks=true]{hyperref}

\begin{document}

\title{Probabilistic Completeness of\\Randomized Possibility Graphs Applied to \\Bipedal Walking in Semi-unstructured Environments}

\author{Author Names Omitted for Anonymous Review. Paper-ID 127}



%

\maketitle

\begin{abstract}
We present a theoretical analysis of a recent whole body motion planning method, the Randomized Possibility Graph \cite{grey2017footstep}, which uses a high-level decomposition of the feasibility constraint manifold in order to rapidly find routes that may lead to a solution. These routes are then examined by lower-level planners to determine feasibility. In this paper, we show that this approach is probabilistically complete for bipedal robots performing quasi-static walking in ``semi-unstructured'' environments. Furthermore, we show that the decomposition into higher and lower level planners allows for a considerably higher rate of convergence in the probability of finding a solution when one exists. We illustrate this improved convergence with a series of simulated scenarios.
\end{abstract}

\IEEEpeerreviewmaketitle

\section{Introduction}

The goal of deploying humanoid robots into complex and challenging terrain motivates us to examine locomotion planning methods that can offer completeness guarantees with maximum autonomy. Locomotion planning methods tend to have various scopes and limitations depending on their underlying algorithms. Some methods can offer guarantees for either completeness or probabilistic completeness within their designated scopes. Other methods might lack proofs of completeness but can perform well in practice.

The classic approach to locomotion planning, also called footstep planning, was introduced by \citet{kuffner2001footstep}\cite{chestnutt2003planning}. In that work, the set of all possible footstep actions is discretized into a finite action set. From there, A* with an admissible heuristic can be used to find a solution which is globally optimal with respect to the action set. When the predetermined action set is sufficient to find a solution, this method is complete. However, the action set determines the branching factor of the search, so having enough actions to satisfy challenging scenarios could make the search intractable.

A mixed-integer optimization approach by \citet{deits2014footstep} decomposes the obstacle-free space into convex regions wherein a bounding box representation of the robot can fit. With these convex regions, the problem of locomotion planning is formulated as a mixed-integer quadratically constrained quadratic program. This provides completeness and global optimality guarantees, as long as the convex regions are sufficient for finding a solution. However, in the implementation that was presented, the robot was not able to duck under obstacles or maneuver its upper body to assist in balance, as would be needed to traverse underneath the bars in Fig. \ref{fig:rho_effect}.

The Multi-modal Probabilistic Roadmap (MMPRM) approach by Hauser et al. \cite{hauser2009multi, hauser2011randomized, hauser2008motion} uses randomized sampling to find whole body motions over a set of predetermined environmental contacts. Given a suitable set of environmental contact points, this method is proven to be probabilistically complete for the full range of quasi-static motion of the robot. This implies that if a quasi-static solution exists, the probability of it being found will asymptotically approach $1.0$ as run time goes to infinity. The key advantage to this approach is that the full kinematic capabilities of the robot are at its disposal, rather than a limited subset.

\begin{figure}
  \centering
  \begin{subfigure}[b]{0.48\linewidth}
    \includegraphics[width=\linewidth]{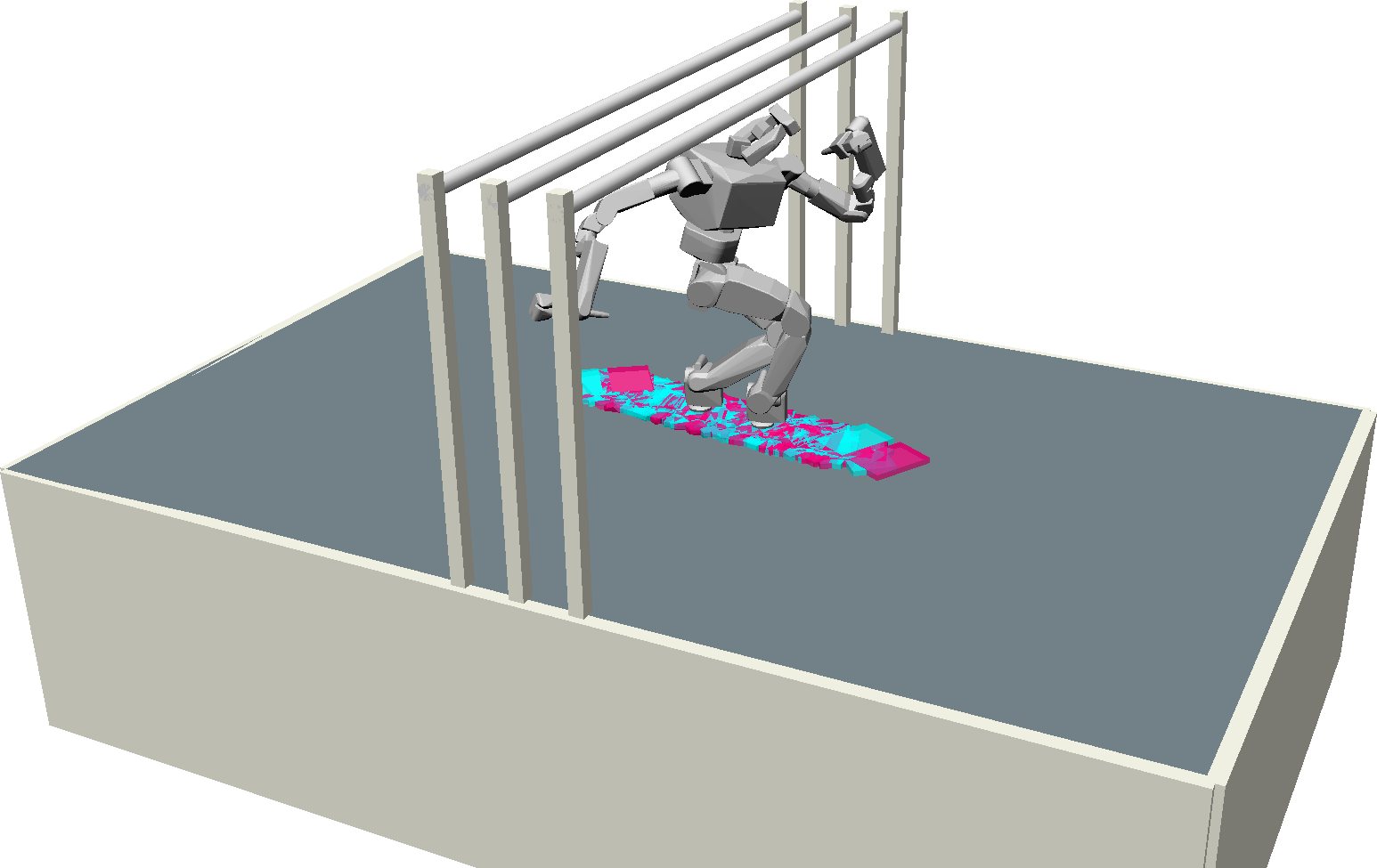}
    \caption{$\rho = 0$}
  \end{subfigure}
  \hfil
  \begin{subfigure}[b]{0.48\linewidth}
    \includegraphics[width=\linewidth]{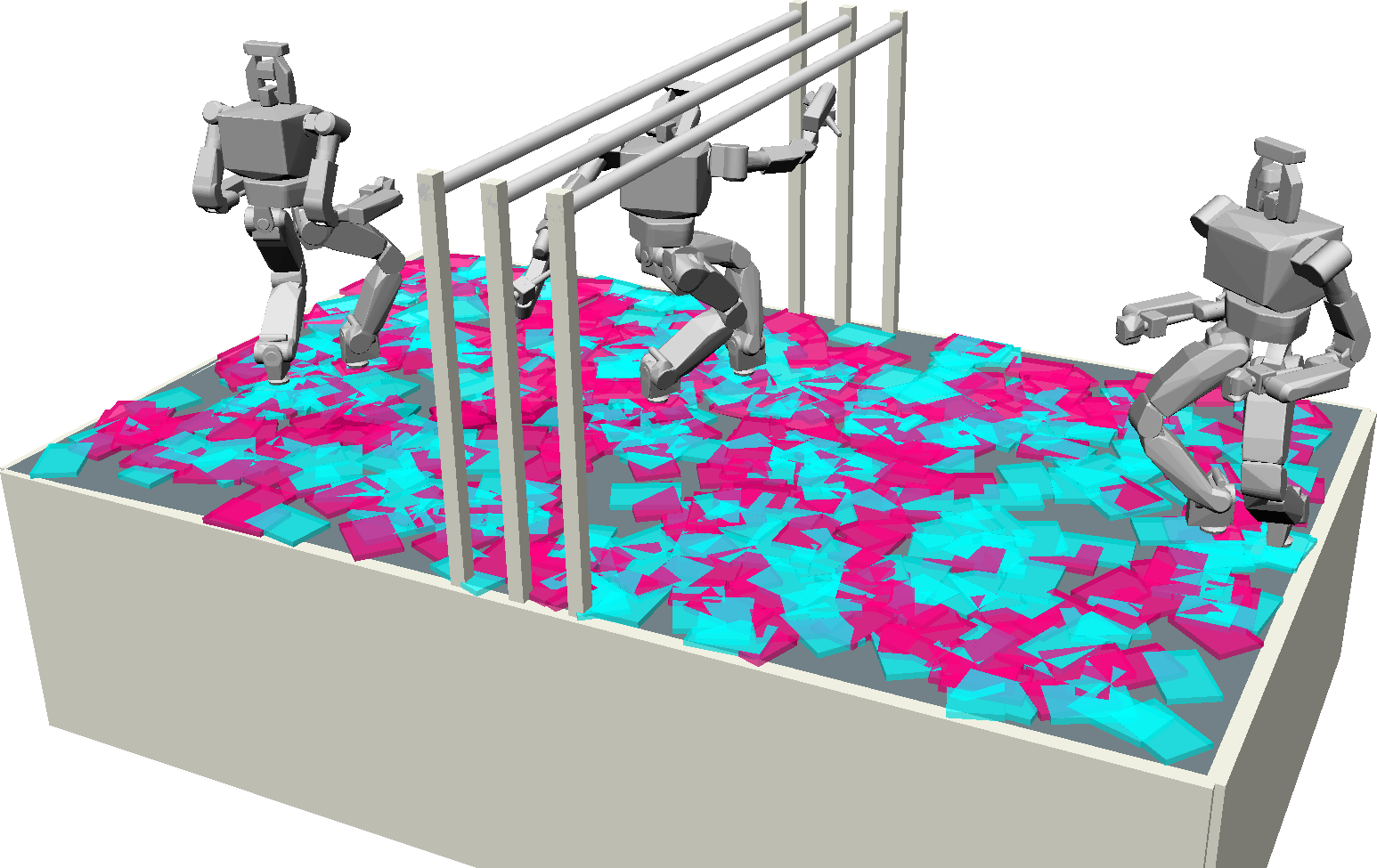}
    \caption{$\rho = 5 \mathcal{R}_\text{max}$}
  \end{subfigure}
  \caption{Illustration of the effect that the parameter $\rho$ has on the Mode Sampling stage. The task for the robot is to pass underneath the set of three bars. Cyan and magenta boxes represent left and right (respectively) foot placement samples, and changing the value for $\rho$ affects the size of the sampling region. $\mathcal{R}_\text{max}$ is the length of the largest step that the robot can take.}
  \label{fig:rho_effect}
\end{figure}

The Randomized Possibility Graph (RPG) method by \citet{grey2017footstep} expands on MMPRM.  The purpose of the RPG is twofold: (1) eliminate the requirement that the planner must be provided with a finite set of environmental contact points, and (2) focus the search effort of the low-level MMPRM into regions that are promising for finding a solution. This is accomplished by first exploring the environment at a high level to find routes that are able to satisfy a set of simplified necessary or sufficient conditions. If a route satisfies the sufficient conditions, a fast and efficient planner can be applied to find a walking motion that follows the route. However, if a route only satisfies the simplified set of necessary conditions, then a low-level multi-modal motion planner must be applied to find viable foot placements and a joint motion that satisfies the full set of feasibility constraints, such as maintaining balance and avoiding obstacles.

In this paper, we examine the probabilistic completeness properties of the RPG method when applied to ``semi-unstructured'' environments. We define ``semi-unstructured'' to mean an environment that contains obstacles with arbitrary geometry, but where the walkable ground is flat and even. The restriction to flat and even ground is due to a limitation in the current implementation of foot placement sampling, which is similar to the Task Space Region approach of \citet{tsr}\cite{berenson2010probabilistically}. This restriction will be loosened in future work, possibly by utilizing a reachable space representation similar to the work of \citet{tonneau2015reachability}.

In the theoretical examination, we consider a simplified algorithm which we will call the Worst-case RPG (\mbox{w-RPG}). The \mbox{w-RPG} exhibits the worst-case behavior of the ordinary RPG algorithm, which is what the ordinary algorithm would degenerate into when none of its built-in performance optimizations are effective. Since the ordinary RPG will have strictly better performance, the analysis of \mbox{w-RPG} represents a lower bound on the worst-case performance of the ordinary algorithm. Therefore, if the \mbox{w-RPG} is proven to be probabilistically complete, then the RPG is as well.

In addition to the proof of probabilistic completeness, we analyze a user-chosen parameter, $\rho$, which affects the rate of convergence. This analysis provides hints for choosing a value for $\rho$ that will provide reliable convergence. We also provide empirical data from simulation trials where the parameter is varied to demonstrate the quantitative impact of this parameter.

\section{Problem Definition}

The work by \citet{grey2017footstep} examined the application and performance of Randomized Possibility Graphs for solving locomotion planning problems in semi-unstructured environments. In this context, ``semi-unstructured'' means the walkable terrain is flat and even, but the environment contains 3D obstacles with arbitrary geometry. Since the obstacle geometry is arbitrary, the robot may require whole body motions in order to maneuver through the environment, like needing to duck underneath the overhanging bars in Fig. \ref{fig:rho_effect}.

\subsection{Probabilistic Completeness}

A process is considered probabilistically complete if the probability of it \textit{failing} to find a solution when one exists converges asymptotically to zero as the number of samples it uses goes to infinity, i.e. the probability of failure can be written as:
\begin{equation}
\Pr[\text{FAILURE}] \leq \alpha \exp(-\beta N)
\end{equation}
where $\alpha$ and $\beta$ are positive constants greater than zero, and $N$ is the number of samples being used by the process.

\subsection{Modes}

Bipedal robots are hybrid dynamic systems (see Ames et al. \cite{ames2011human, reher2016algorithmic} for examples of detailed hybrid system models for bipeds) which exhibit sequences of discrete \textit{modes}. In the scope of this paper, a mode is defined by the placement of the support foot (or feet, in the case of double-support modes). Each mode corresponds to a set of feasibility constraints which determine whether a given configuration is physically viable for that mode. A mode takes the following form:
\begin{displaymath}
\sigma = 
  \begin{cases}
    x_f \in \mathbb{R}^3 \mid f \in \{\text{Left}, \text{Right}\} & \quad \text{if Single-Support}\\
    \left(x_\text{Left} \in \mathbb{R}^3, x_\text{Right} \in \mathbb{R}^3\right) & \quad \text{if Double-Support}\\
  \end{cases}
\end{displaymath}
where $x_f$ represents a foot placement consisting of two translational dimensions and one rotational dimension for the yaw of the foot.

\subsection{Worst-case Randomized Possibility Graph}

%

\begin{figure}
  \centering
  \includegraphics[width=0.9\linewidth]{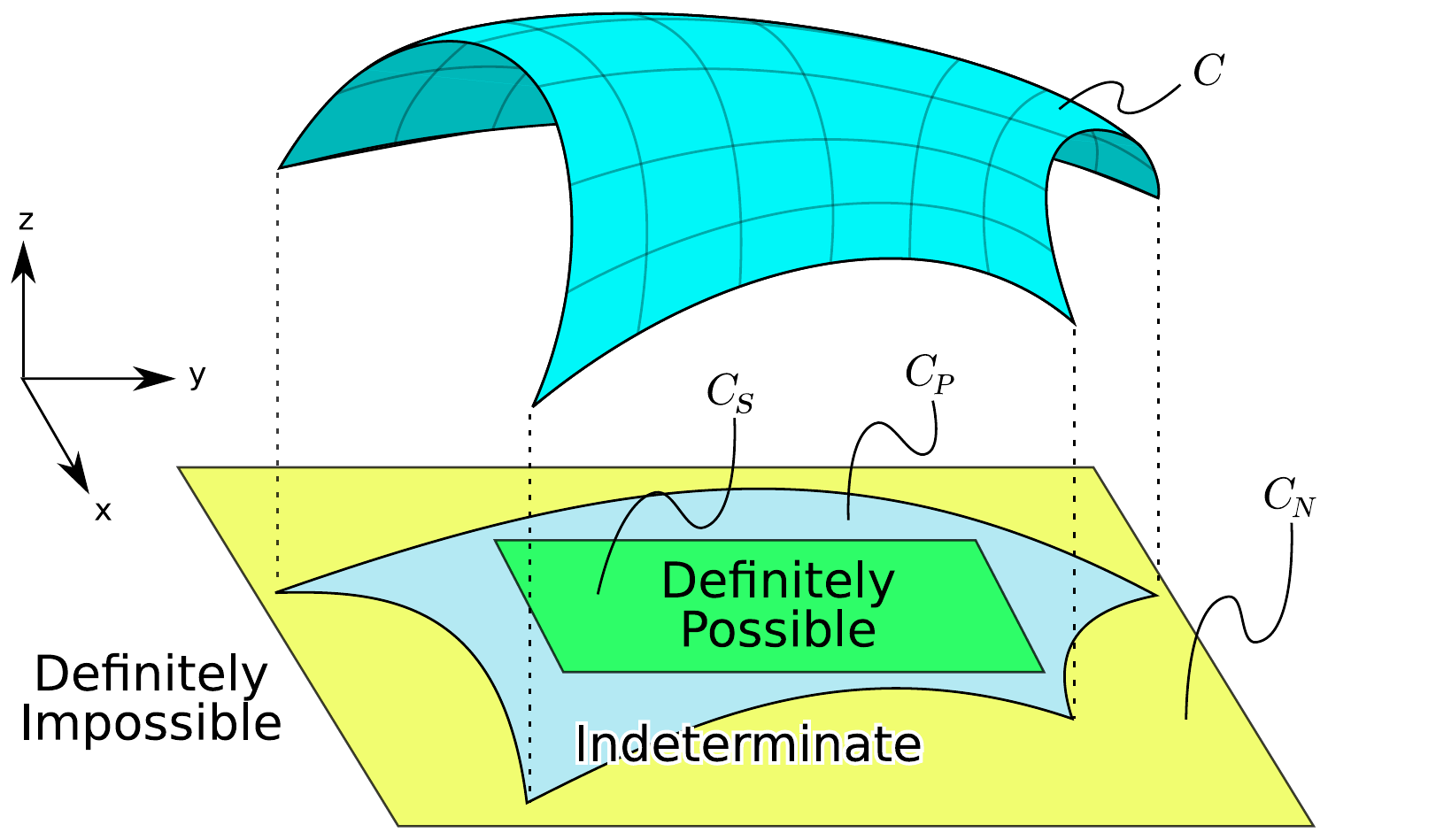}
  \caption{Abstract depiction of the difference between sufficient (green) vs. necessary (yellow) conditions. The constraint manifold $C$ is projected down to a lower-dimensional manifold $C_P$. The sufficient condition manifold, $C_S$ is a simple shape which fits entirely inside of $C_P$ while the necessary condition manifold, $C_N$, is a simple shape which contains all of $C_P$.}
  \label{fig:nec_vs_suf}
\end{figure}

For the analysis of this paper, we consider a simplified version of the RPG scheme which we will call \mbox{w-RPG}. The simplified version discards the use of sufficient conditions (see Fig. \ref{fig:nec_vs_suf}) and the performance benefits that come with them, so performance of \mbox{w-RPG} represents the worst-case performance of RPG. Using only the necessary conditions allows for more straightforward theoretical analysis and helps to establish an upper bound on the probability of failing to find a solution when one exists. There are three stages to \mbox{w-RPG}:

\paragraph{Possibility Exploration}

Instead of growing trees, we explore possibilities by sampling $N_{P}$ points in the Possibility Exploration Space $\mathscr{E}$, which in this context is SE(3). Within the manifold of $\mathscr{E}$ is a submanifold called $C_N$ which represents the points in $\mathscr{E}$ where the simplified set of necessary conditions are satisfied. Any random samples which are not inside of $C_N$ are rejected from the sample set. We then perform an $\mathcal{O}({N_{P}}^2)$ operation attempting to connect every pair of points with a ``straight line''. When a route through $C_N$ is found that might be able to connect the start and goal states, this route is sent to the next stage: Mode Sampling.

\begin{figure}
  \centering
  \subcaptionbox{An example 3D representation of the space that can be reached by the left foot from a fixed location of the right foot (magenta arrow).\label{fig:3D_reachability}}{\includegraphics[width=0.45\linewidth]{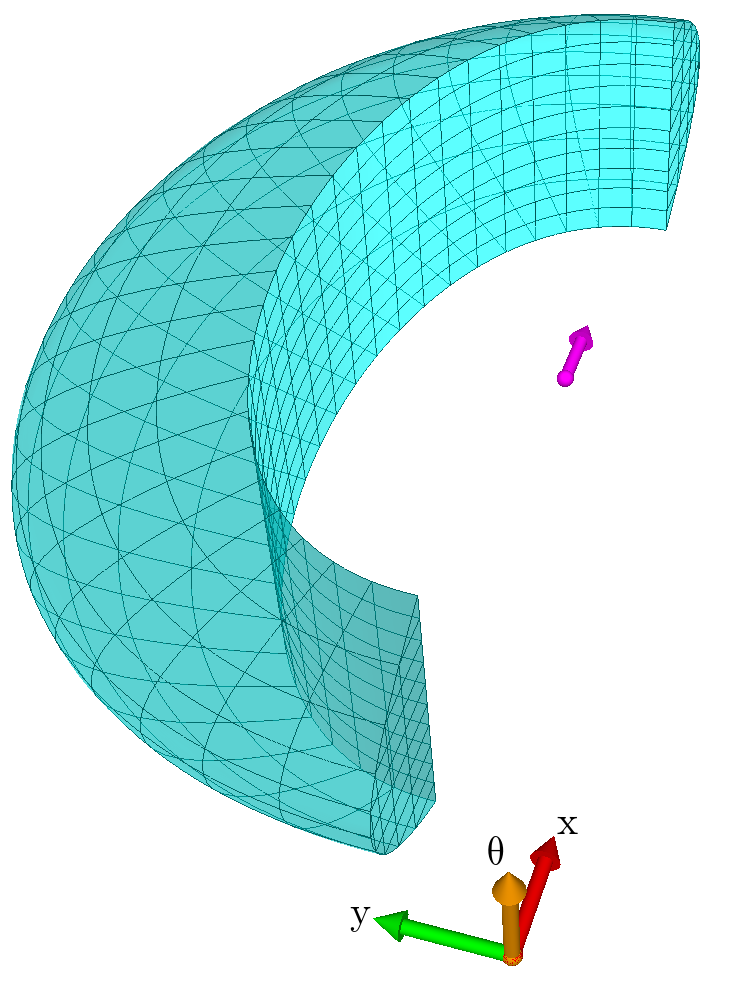}}
  \hfil
  \subcaptionbox{\label{fig:3D_intersection}An example of how the reachable spaces of two different foot locations can intersect. The interior cylinder is a subset of the left-foot locations that can be reached from both right-foot locations.\label{fig:3D_intersection}}{\includegraphics[width=0.45\linewidth]{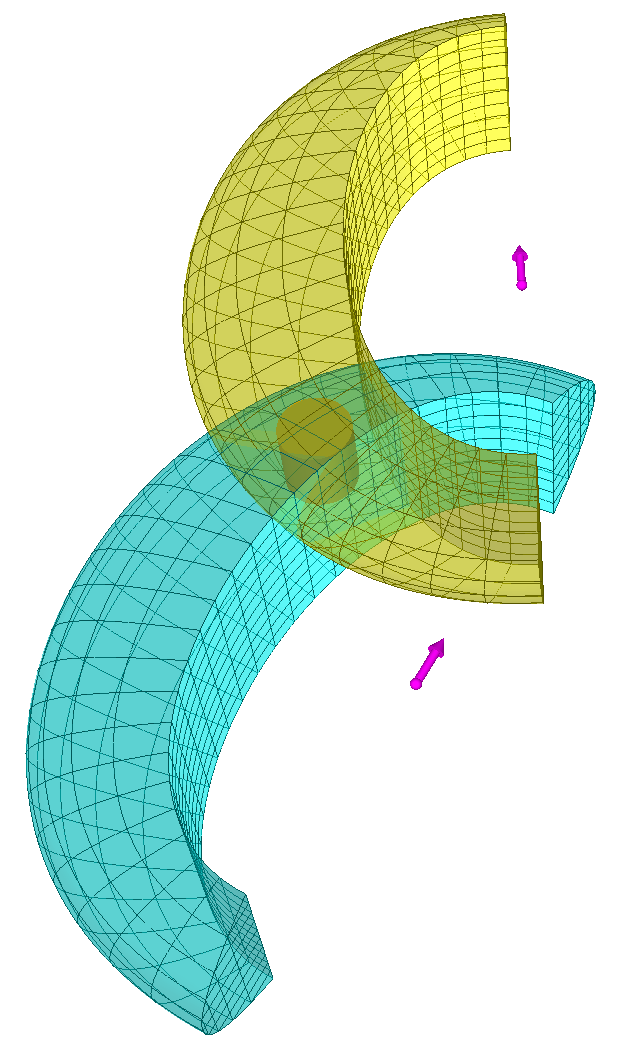}}
  \caption{3D illustrations of what a reachable space might look like for a bipedal system. Magenta arrows represent right-foot placements. Cyan and yellow regions are the corresponding reachable left-foot locations. The axes represent x/y translation and yaw.}
\end{figure}

\paragraph{Mode Sampling}
\label{sec:mode_sampling}
We sample modes uniformly near the route produced by the Possibility Exploration stage. The elements of the route are projected from SE(3) to $\mathbb{R}^2$, keeping only the $(x,y)$ values from the route. Then $N_\sigma$ left and right foot placements are uniformly sampled within a radius $\rho$ of each projected vertex along the route. The union of these circles is referred to as $\mathcal{F}_\sigma$. The orientations of the foot placements are uniformly sampled from $[0, 2\pi)$.

Once the foot placements are sampled, we perform an $\mathcal{O}({N_\sigma}^2)$ operation to test whether each pair of foot placements can reach each other. Each foot placement is assigned a single-support mode based on whether it is viable as a left- or right-foot placement. Each pair of foot placements that can reach each other are assigned a double-support mode.

\paragraph{Multi-modal PRM}

Once a discrete set of modes have been sampled, Multi-modal PRM as described by \citet{hauser2008motion} is used to find valid whole body paths through the modes.

\section{Proof of Completeness}

To prove the probabilistic completeness of the overall procedure, we first prove the probabilistic completeness of the Mode Sampling stage, and then prove the probabilistic completeness of the Possibility Exploration stage. Multi-modal PRM is already proven to be probabilistically complete by \citet{hauser2009multi}, so the final step is to show that the product of multiple dependent probabilistically complete processes is also probabilistically complete.

\subsection{Completeness of Mode Sampling}

\begin{figure}
  \centering
  \begin{subfigure}[b]{0.22\linewidth}
    \captionsetup{justification=centering}
    \centering
    \includegraphics[width=\linewidth]{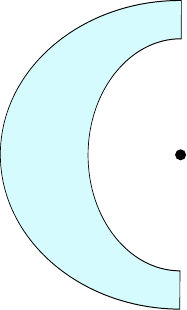}
    \caption{}
    \label{fig:left_reachability}
  \end{subfigure}
  \hfil
  \begin{subfigure}[b]{0.22\linewidth}
    \captionsetup{justification=centering}
    \centering
    \includegraphics[width=\linewidth]{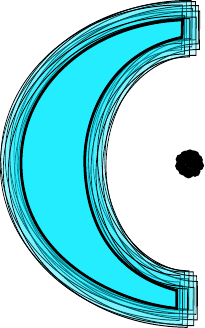}
    \caption{}
    \label{fig:contraction_proof_samples}
  \end{subfigure}
  \hfil
  \begin{subfigure}[b]{0.22\linewidth}
    \captionsetup{justification=centering}
    \centering
    \includegraphics[width=\linewidth]{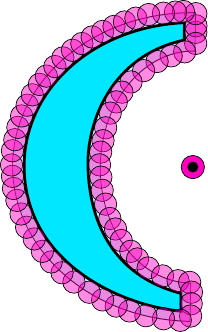}
    \caption{}
    \label{fig:contraction_proof_circles}
  \end{subfigure}
  \caption{(a) A slice, $s$, of the reachable area for the left foot when the right foot is at the black dot. (b) Samples of the set $S$ created by translating $s$ around within a small radius. (c) The shape $\cap S$ created by the intersection of all elements within $S$. This is the same as the original shape, but contracted by circles around the border whose radii are equal to the maximum radius of the translations.}
  \label{fig:contraction_proof}
\end{figure}

To have a viable sequence of modes, each mode in the sequence must be \textit{adjacent} to the mode that comes before and after it. For two modes to be adjacent, their feasible spaces must intersect. A quick way to test for adjacency is to consider the kinematic reachability of one foot with respect to the other foot. For flat and even terrain, the reachable space is a function of the (x,y) position and yaw, $\theta$, of the support foot. An illustration of what such a space might look like can be found in Fig. \ref{fig:3D_reachability}. We assume that the reachable space is a subset of SE(2), containing at least one ball of radius $\epsilon > 0$.

\label{sec:mode_sampling}
For a sequence of modes to be valid, the foot placement of each single-support mode must be simultaneously reachable from the single-support modes that come before and after it, like the cylinder shown in Fig. \ref{fig:3D_intersection}. The following lemma will help us show that there exists a region of foot placements wherein every placement is reachable from \textit{every} member of a region of placements of the other foot.

\begin{lemma}
\label{lem:contraction}
Suppose we have a 2D shape, $s$ (Fig. \ref{fig:left_reachability}). Consider the set of all possible translations of this shape within a circle of fixed radius $r$, $S = \{\sigma \in \mathrm{Trans}(s,\mathbf{x}) \mid |\mathbf{x}| < r\}$ where $\mathrm{Trans}(s,\mathbf{x})$ translates the shape $s$ by vector $\mathbf{x}$ (Fig. \ref{fig:contraction_proof_samples}).

Then the shape of the intersection of all elements in $S$, $\cap S$, is equal to the shape of $s$ contracted by circles of radius $r$ densely packed around its border (Fig. \ref{fig:contraction_proof_circles}).
\end{lemma}

\begin{proof}
Define $b_s$ to be the boundary of $s$. The elements of $s$ can be divided into two sets: $\alpha = \{x \in s \mid d(x, b_s) \geq r\}$ and $\beta = \{x \in s \mid d(x, b_s) < r\}$ where $d(x,b_s)$ computes the smallest distance between $x$ and $b_s$.

An element $x \in s$ will \textbf{not} exist in the shape of $\cap S$ if and only if at least one shape in $S$ was transformed by a distance greater than $d(x,b_s)$. Otherwise $x$ cannot be outside the border of any shape in $S$.

By definition, the elements $x \in \alpha$ have the property $d(x, b_s) \geq r$, and every element of $S$ was translated by less than $r$, so all of the elements of $\alpha$ must remain in $\cap S$.

Conversely, the elements $x \in \beta$ have the property $d(x, b_s) < r$. Since $S$ contains elements which have been transformed by a distance up to $r$ in every direction, the elements of $\beta$ cannot remain in $\cap S$. Moreover, the elements of $\beta$ are the same elements that would be covered by circles of radius $r$ which are densely packed around $b_s$. An illustration of this effect can be seen in Fig. \ref{fig:contraction_proof}.
\end{proof}

The effect of lemma \ref{lem:contraction} can be generalized to the 3D shape of Fig. \ref{fig:3D_reachability} by continuously applying it to slices along the $\theta$ axis. If the original shape represented the space that is reachable from $x_f$, then the contracted shape would then represent the set of foot locations that can be reached from \textit{any} location within a cylinder centered around $x_f$.

Now we can derive an upper bound on the probability of failing to sample a set of modes that can enable the system to reach the goal from the start, if such a set of modes exists. Assume there exists \emph{some} solution, which is a function that outputs a configuration and a mode as a function of time:
\begin{displaymath}
\gamma_S : [0, t_f] \longmapsto \mathbb{R}^{N_C} \times \Sigma
\end{displaymath}
where $N_C$ is the size of the configuration space and $\Sigma$ is the set of all possible modes. The configuration output will vary continuously, but the sequence of modes through $[0, t_f]$ will be discrete and finite. Figure \ref{fig:solution_steps} displays an environment with the foot placements of a hypothetical solution that allows the robot to traverse from the bottom left to top right. We will now show that this selection of foot placements is not unique, and that uniform random sampling is a probabilistically complete way of finding a suitable sequence of modes to connect the start and the goal states.

\begin{figure}
  \centering
  \begin{subfigure}[t]{0.7\linewidth}
    \captionsetup{justification=centering}
    \centering
    \includegraphics[width=\linewidth]{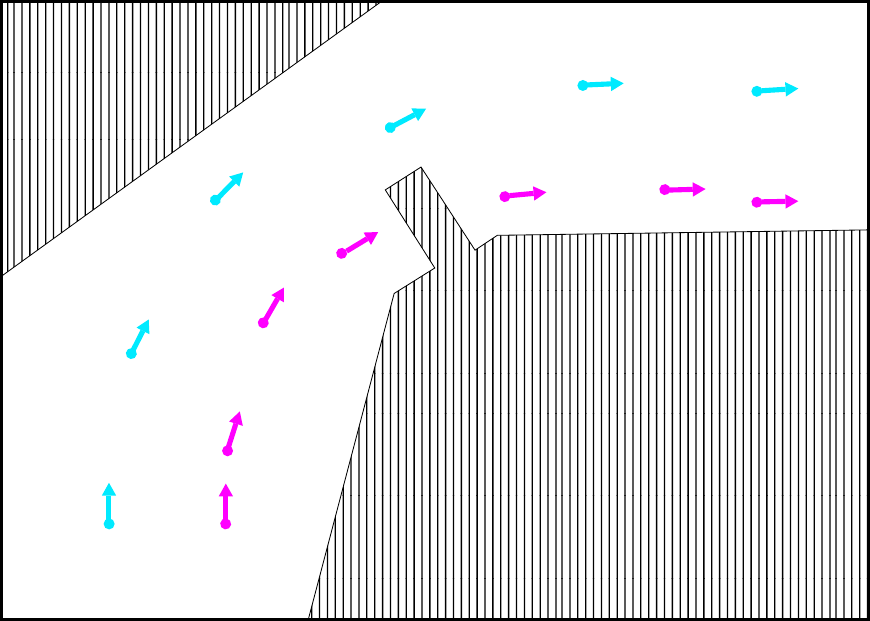}
    \caption{Cyan and magenta arrows represent the foot placements of a hypothetical solution}
    \label{fig:solution_steps}
  \end{subfigure}
  \hfil
  \begin{subfigure}[t]{0.7\linewidth}
    \captionsetup{justification=centering}
    \centering
    \includegraphics[width=\linewidth]{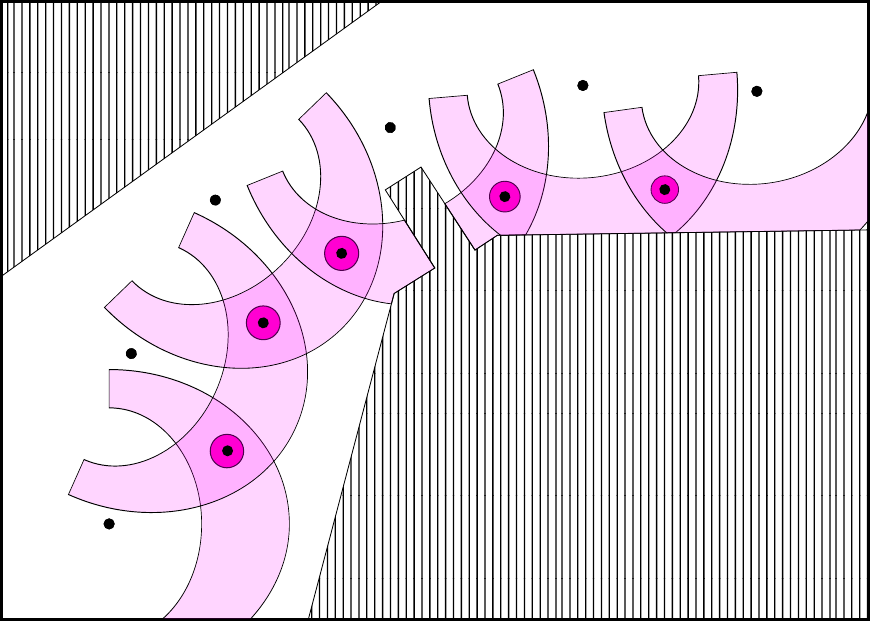}
    \caption{Magenta regions represent areas that the right foot can reach for each given left foot placement.}
    \label{fig:solution_right_reachability}
  \end{subfigure}
  
  \begin{subfigure}[t]{0.7\linewidth}
    \captionsetup{justification=centering}
    \centering
    \includegraphics[width=\linewidth]{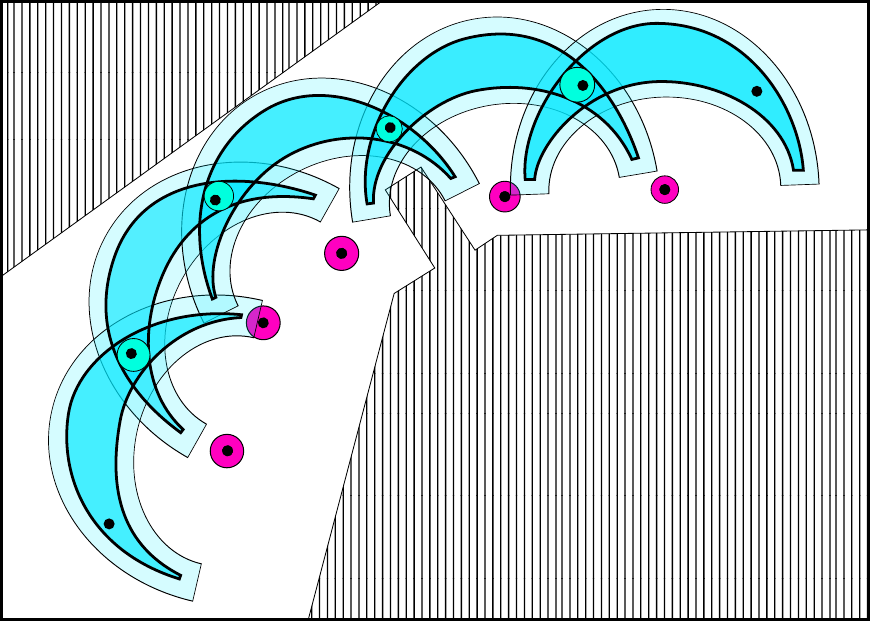}
    \caption{Dark teal regions represent areas that the left foot can reach from any right foot placement within each magenta ball. The lighter teal border shows the original reachable shape, before being contracted.}
    \label{fig:solution_left_reachability}
  \end{subfigure}
  \hfil
  \begin{subfigure}[t]{0.7\linewidth}
    \captionsetup{justification=centering}
    \centering
    \includegraphics[width=\linewidth]{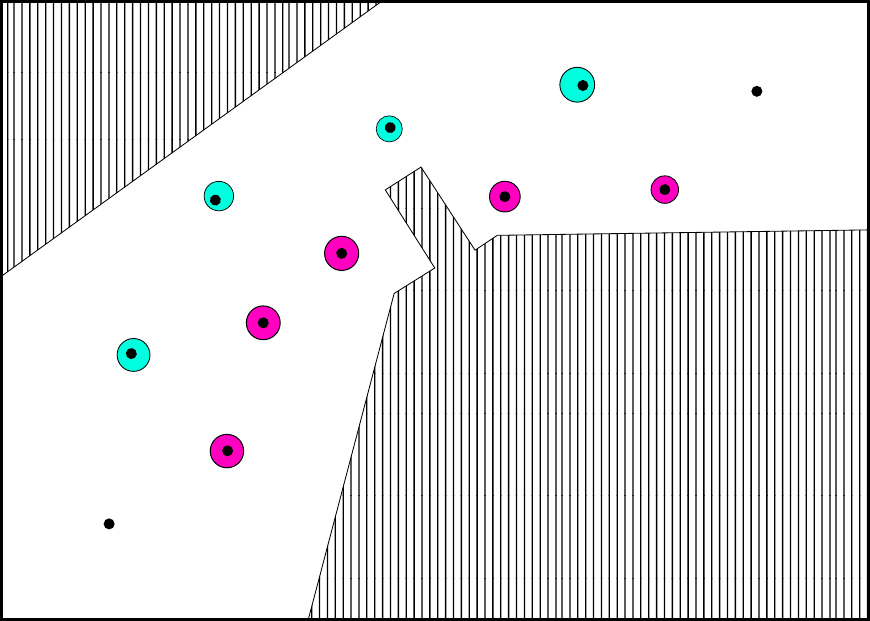}
    \caption{Each ball represents the foot placements that can be reached from any foot placement within the previous and next ball.}
    \label{fig:solution_mode_regions}
  \end{subfigure}
  \caption{An environment consisting of regions where foot placements are valid (white) and invalid (striped). Foot placements may be invalid due to holes in the ground or obstacles on the ground.}
  \label{fig:solution_modes}
\end{figure}

\begin{theorem}
\label{thm:mode_complete}
Let there be a sequence of $M$ single-support modes $\{\sigma_1, ..., \sigma_{M}\}$ that are sufficient to connect a start state $x_\text{start} = (q_\text{start}, \sigma_\text{start})$ to a goal state $x_\text{goal} = (q_\text{goal}, \sigma_\text{goal})$. (Note that double-support modes exist between the single-support modes within the solution, but the double-support modes are not relevant to this theorem.)

Then the probability that $N_\sigma$ uniform samples of placements for each foot will fail to find a set of modes that can connect $x_\text{start}$ to $x_\text{goal}$ is at most
\begin{equation}
\label{eqn:mode_theorem}
M \left(1 - \beta_m \right)^{N_\sigma}
\end{equation}
where $0 < \beta_m \leq 1$ is a problem-dependent constant.


\end{theorem}

\begin{proof}
For a sequence of alternating single-support modes $\{\sigma_1, ..., \sigma_{M}\}$ to be valid, it is necessary for $\sigma_{i+1}$ to be reachable from $\sigma_i$. Moreover, due to the symmetry of reachability, it is also necessary for $\sigma_i$ to be reachable from $\sigma_{i+1}$.

The space of foot placements that are reachable from $\sigma_i$ is given by the set $\mathcal{R}(\sigma_i)$. Therefore, for each mode $\sigma_{2i+1}, i=0,...,\floor{\frac{M-1}{2}}$ we can identify a range of alternative foot placements by taking the intersection $\Sigma_{2i+1} = \mathcal{R}(\sigma_{2i}) \cap \mathcal{R}(\sigma_{2i+2})$. The foot placement for $\sigma_{2i+1}$ can be replaced by any element in $\Sigma_{2i+1}$ without affecting the validity of the mode sequence, because all elements in $\Sigma_{2i+1}$ are reachable from the modes that come both before and after $\sigma_{2i+1}$. Examples of $\Sigma_{2i+1}$ can be seen in the overlapping magenta regions of Fig. \ref{fig:solution_right_reachability}.

Let us construct a cylinder named $\varsigma_{2i+1}$ of radius $r_{2i+1}$ within each $\Sigma_{2i+1}$ for $i=0,...,\floor{\frac{M-1}{2}}$ (see Fig. \ref{fig:3D_intersection} for a 3D illustration of such a cylinder, and Fig. \ref{fig:solution_right_reachability} for an overhead view of a sequence of cylinders). For each $\varsigma_{2i+1}$, the set of foot placements which are reachable by every member of the cylinder will be $\cap \mathcal{R}(\varsigma_{2i+1})$. From Lemma \ref{lem:contraction}, we know that the shape of this intersection will be the ordinary shape of reachability but contracted by circles of $r(\varsigma_{2i+1})$ densely packed around the border. These contracted regions are illustrated in Fig. \ref{fig:solution_left_reachability}. The cylinder also has a height, ${\Delta \theta}_{2i+1}$, which is chosen in conjuncture with $r_{2i+1}$ such that the cylinder fits inside of $\Sigma_{2i+1}$.

Now for $i=1,...,\ceil{\frac{M-1}{2}}$ choose the largest cylinder available within the intersection $\{\cap \mathcal{R}(\varsigma_{2i-1})\} \cap \{\cap \mathcal{R}(\varsigma_{2i+1})\}\}$ and call it $\varsigma_{2i}$. Note that $\sigma_0$ and $\sigma_{M+1}$ are the start and goal (respectively) single-support modes which are given by the problem query. It is sufficient to have $\varsigma_0 \equiv \{\sigma_0\}$ and $\varsigma_{M+1} \equiv \{\sigma_{M+1}\}$, because both of those modes are provided without any sampling.

We now have a sequence of cylinders $\varsigma_i, i=1,...,M$ where as long as at least one foot placement from each cylinder is sampled, the set of samples will be sufficient for finding a valid solution that connects the start and goal states. Each cylinder is defined by its radius, $r_i$ and height, ${\Delta \theta}_i$. These parameters would ideally be chosen such that they maximize the volume of the smallest cylinder in the set. Choose $r_m$ and ${\Delta \theta}_m$ to be the radius and height of the cylinder with minimal volume. The volume of this minimal cylinder is then $\pi r_m^2 {\Delta \theta}_m$.

Suppose we are given a planar region to sample from, $\mathcal{F}_\sigma$. Yaw values can simply be sampled from the range $[0, 2\pi]$. This gives us a sampling volume of $2\pi|\mathcal{F}_\sigma|$. If the $x$/$y$ translations of the foot placements within each $\varsigma_i$ all lie in $\mathcal{F}_\sigma$, and we take $N_\sigma$ independent samples of left-support modes and $N_\sigma$ samples of right-support modes from $\mathcal{F}_\sigma$, then we get

\begin{displaymath}
\begin{split}
\Pr[\text{FAILURE}] & \leq \Pr[\text{Some cylinder $\varsigma_i$ is not sampled}]\\
  & \leq \sum_{i=1}^{M} \Pr[\text{Cylinder $\varsigma_i$ is not sampled}]\\
  & = \sum_{i=1}^{M} \left(1 - \frac{\pi r_i^2 {\Delta \theta}_i}{2\pi|\mathcal{F}_\sigma|}\right)^{N_\sigma}\\
  & \leq M \left(1 - \frac{\pi r_m^2 {\Delta \theta}_m}{2\pi |\mathcal{F}_\sigma|}\right)^{N_\sigma}
\end{split}
\end{displaymath}
which gives us
\begin{equation}
\label{eqn:mode_complete}
\Pr[\text{FAILURE}] \leq M \left(1 - \frac{r_m^2 {\Delta \theta}_m}{2|\mathcal{F}_\sigma|}\right)^{N_\sigma}
\end{equation}

If we then take
\begin{displaymath}
\beta_m = \frac{r_m^2 {\Delta \theta}_m}{2 |\mathcal{F}_\sigma|}
\end{displaymath}
we know that $0 < \beta_m \leq 1$ because $r_m$ and ${\Delta \theta}_m$ are non-zero (except in pathological cases), and the volume of the sampling space must be at least as large as the volume of the smallest cylinder in order to satisfy the assumption that $\mathcal{F}_\sigma$ covers all foot placements in each set $\varsigma_i$. Therefore, substituting $\beta_m$ into equation \ref{eqn:mode_complete} gives us the expression in equation \ref{eqn:mode_theorem}.
\end{proof}

\subsection{Completeness of Possibility Exploration}

Now we consider the Possibility Exploration stage, where we find samples that exist in the necessary condition manifold, $C_N$, and connect them in a graph using geodesics (\citet{kuffner2004effective} provides useful implementation details for sampling points in SE(3) and connecting them). We derive an upper bound for the probability that $N_P$ samples will fail to provide a route that can be used by the Mode Sampling stage to find adequate mode samples for a solution. This proof is largely derived from the proofs of probabilistic completeness presented by \citet{kavraki1998analysis} and \citet{svestka1996probabilistic}, but we also account for the need to obtain an adequate sampling of foot placements, which was not a requirement for prior proofs.

\begin{figure}
  \centering
  \includegraphics[width=\linewidth]{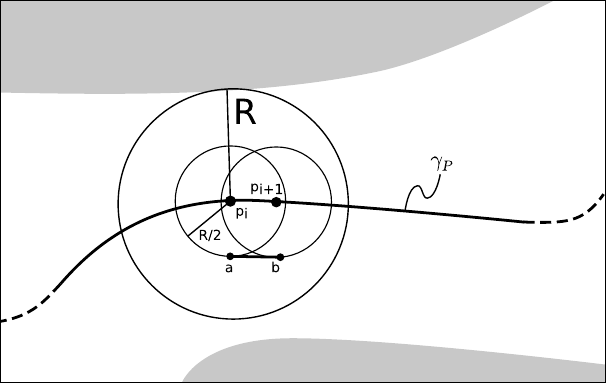}
  \caption{Illustration for the proof of Lemma \ref{lem:R_complete}. The white area represents $C_N$ while gray is $\mathscr{E}\setminus C_N$. $R$ is the minimum distance between the path $\gamma_P$ and the edge of $C_N$.}
  \label{fig:collision_avoidance}
\end{figure}

As before, assume a solution exists in the form:
\begin{displaymath}
\gamma_S : [0, t_f] \longmapsto \mathbb{R}^{N_C} \times \Sigma
\end{displaymath}
We can transform this function into
\begin{displaymath}
\gamma_P(s(t)) = \text{Proj}_{\mathscr{E}}(\gamma_S(t))
\end{displaymath}
where $\text{Proj}_{\mathscr{E}}(x)$ is a function that projects a state $x$ into the Possibility Exploration Space, and $s(t)$ parameterizes $\gamma_P$ by arclength instead of time.

\textbf{Definitions:} We denote $d_{\gamma_P}(s,r)$ to compute the arclength distance between points $\gamma_P(s)$ and $\gamma_P(r)$ along the curve $\gamma_P$. We define $B_r(s)$ to be the set of all points in $\mathscr{E}$ within a ball of radius $r$ centered at $\gamma_P(s)$. Recall that $\mathscr{E}$ is the ``Exploration Space'' from which we randomly sample points to see if they satisfy the necessary conditions. In the context of this paper, $\mathscr{E}$ is equal to SE(3) where the translational dimensions are bounded by a box.

%
%
%

\begin{lemma}
\label{lem:R_complete}
Let $\gamma_P : [0, L] \longrightarrow \mathscr{E}$ be a path that connects $p_\text{start} = \text{Proj}_{\mathscr{E}}(x_\text{start})$ and $p_\text{goal} = \text{Proj}_{\mathscr{E}}(x_\text{goal})$. Let $R = \inf_{0 \leq s \leq L} r(\gamma_P(s))$ be the minimum distance of the path to the edge of the necessary condition manifold $C_N$.

Then the probability that $N_P$ uniform samples of $C_N$ will fail to yield a path that can connect from $p_\text{start}$ to $p_\text{goal}$ is no greater than
\begin{equation}
\label{eqn:R_complete}
\frac{L}{\varepsilon}{\left(1 - \frac{\pi^3 \varepsilon^6}{6|C_N|}\right)}^{N_P}
\end{equation}
where $0 < \varepsilon \leq R/2$, and $|C_N|$ is the volume of the necessary condition manifold.

\end{lemma}

\begin{proof}
Let $n = \ceil{L/\varepsilon}$. We can then find a set of points $\{p_0 = p_\text{start}, p_1, ..., p_n = p_\text{goal} \in \gamma_P \mid  \forall i, d_{\gamma_P}(p_i, p_{i+1}) \leq \varepsilon\}$. Note that
\begin{equation}
\label{eqn:balls}
B_{R/2}(p_{i+1}) \subseteq B_{R}(p_i), \text{for } i=0,...,n-1.
\end{equation}
This follows from the triangle inequality and the inequality $|\gamma_P(s)-\gamma_P(r)| \leq d_{\gamma_P}(s,r)$. Assume we have the points $a \in B_{\varepsilon}(p_i)$ and $b \in B_{\varepsilon}(p_{i+1}))$. If we enforce $\varepsilon \leq R/2$, then $B_{\varepsilon}(p_i) \subseteq B_{R/2}(p_i)$, and equation \ref{eqn:balls} guarantees that both $a,b \in B_{R}(p_i)$.  Therefore, there is guaranteed to be a geodesic line segment $\overline{ab}$ that lies entirely within $C_N$ and connects the points $a$ and $b$, because every point in $B_R(p_i)$ lies within $C_N$ due to the definition of $R$. This property is illustrated in Fig. \ref{fig:collision_avoidance}.

This observation tells us that it is sufficient to have at least one sample point in each ball $B_{\varepsilon}(p_i), i=1,...,n-1$ for the Possibility Exploration stage to find a path that connects the start point to the goal point, as long as $\varepsilon \leq R/2$. We can sample SE(3) from $\mathbb{R}^6$ without loss of generality using an Euler angle representation of orientation. Therefore the volume of the balls to be sampled can be computed based on a 6-ball: $\pi^3 \varepsilon^6/6$. Taking $N_P$ independent samples from $C_N$, we find
\begin{equation}
\label{eqn:R_proof}
\begin{split}
\Pr[\text{FAILURE}] & \leq \Pr[\text{Some ball is not sampled}]\\
  & \leq \sum_{i=1}^{n-1} \Pr[\text{Ball $B_{\varepsilon}(p_i)$ is not sampled}]\\
  & \leq \frac{L}{\varepsilon}{\left(1 - \frac{\pi^3 \varepsilon^6}{6|C_N|}\right)}^{N_P}
\end{split}
\end{equation}
\end{proof}

\textbf{Definition:} $d_{xy}(\sigma, p)$ computes the distance across the $xy$-plane between the foot location corresponding to the mode $\sigma$ and the point $p$.

\begin{lemma}

\label{lem:rho_complete}
As in Lemma \ref{lem:R_complete}, $\gamma_P : [0, L] \longrightarrow \mathscr{E}$ is a path that connects $p_\text{start}$ and $p_\text{goal}$. Let $h_m$ represent the greatest distance of any foot placement in the union $\cup \varsigma_i, i=1,...,M$ from the point on the path $\gamma_P$ which is closest to that mode:
\begin{equation}
\label{eqn:h_m}
h_m = \sup_{\sigma \in \cup_i \varsigma_i} \inf_{s \in [0,L]} d_{xy}(\sigma, \gamma_P(s))
\end{equation}
Given a value of $\rho \geq 2h_m$ (see Fig. \ref{fig:rho}), the probability that $N_P$ uniform samples of $C_N$ will not be adequate to sample the modes needed for a solution is no greater than
\begin{equation}
\frac{L}{\varepsilon}{\left(1 - \frac{\pi^3 \varepsilon^6}{6|C_N|}\right)}^{N_P}
\end{equation}
where $0 < \varepsilon \leq \rho/4$ and $|C_N|$ is the volume of the necessary condition manifold.
\end{lemma}

\begin{figure}
  \centering
  \includegraphics[width=0.86\linewidth]{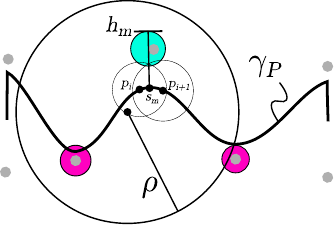}
  \caption{Illustration of the parameter $h_m$. Teal and magenta balls represent the cylindrical regions of acceptable foot placements, $\varsigma_i$, from Sec. \ref{sec:mode_sampling}, and gray dots represent the foot placements that are used by the hypothetical solution of $\gamma_S$. Small black dots are points in the Possibility Exploration Space $\mathscr{E}$.}
  \label{fig:rho}
\end{figure}

\begin{proof} As in the proof for Lemma \ref{lem:R_complete}, let us define a set of $n=\ceil{L/\varepsilon}$ points $p_0=p_\text{start},p_1,...,p_n=p_\text{goal}$ along $\gamma_P$ such that $d_{\gamma_P}(p_i,p_{i+1}) \leq \varepsilon$ for each $i=0,...,n-1$. 

Suppose we choose $\rho$ such that $\rho \geq 2 h_m$. This condition is easily enforced using known information by setting $\rho$ to be at least double the furthest distance that the robot can step, which we will refer to as $\mathcal{R}_\text{max}$. Additionally, suppose we enforce $\varepsilon \leq \rho/4$. Define $s_m$ to be the minimizer for $s$ in equation \ref{eqn:h_m}. Define $p_m$ to be the point from the set $\{p_0, ..., p_{n-1}\}$ that is closest to the value $s_m$. We know that $p_m$ cannot be further than $\rho/4$ from $s_m$, or else another ball would have been placed in the sequence, and that new ball would be closer to $s_m$ than $p_m$, which would contradict the definition of $p_m$. Therefore, no point in $B_{\rho/4}(p_m)$ can be further from $\gamma_P(s_m)$ than $\rho/2$.

Define $\sigma_m$ to be the maximizer for $\sigma$ in equation \ref{eqn:h_m}. If $x_m$ is the translational location of the foot placement for $\sigma_m$, then $x_m$ has a distance $h_m$ from $\gamma_P(s_m)$. Therefore, the triangle inequality tells us that the furthest distance that $x_m$ could possibly have from $p_m$ is $\delta_m \leq \rho/2 + h_m \leq \rho$. Since $x_m$ is the furthest possible foot placement, all other foot placements in the union of $\varsigma_1, ..., \varsigma_M$ must be within a distance $\delta_j \leq \delta_m \leq \rho$ of every point within some ball $B_{\varepsilon}(p_i), i=0,...,n-1$, as long as $\varepsilon \leq \rho/4$. Figure \ref{fig:rho} illustrates this property.

Therefore, as long as $\rho \geq 2h_m$ and $\varepsilon \leq \rho/4$, it is sufficient to have at least one sample point in each ball $B_\varepsilon(p_i), i=1,...,n-1$ for $\mathcal{F}_\sigma$ in the Mode Sampling stage (see Sec. \ref{sec:mode_sampling}) to cover all the modes of $\varsigma_1,...,\varsigma_M$. The probability of failing to sample each ball at least once is no greater than
\begin{displaymath}
\frac{L}{\varepsilon}{\left(1 - \frac{\pi^3 \varepsilon^6}{6|C_N|}\right)}^{N_P}
\end{displaymath}
\end{proof}
\begin{theorem}
\label{thm:pg_complete}
Let $\gamma_P : [0, L] \longrightarrow \mathscr{E}$ be a path that connects $p_\text{start}$ and $p_\text{goal}$. Given $R$ as defined by Lemma \ref{lem:R_complete}, $h_m$ as defined by equation \ref{eqn:h_m}, and $\rho \geq 2h_m$, the probability that $N_P$ uniform samples of $C_N$ will fail to yield a path that can lead to a solution is no greater than
\begin{equation}
\label{eqn:pg_complete}
\frac{L}{\varepsilon}{\left(1 - \frac{\pi^3 \varepsilon^6}{6|C_N|}\right)}^{N_P}
\end{equation}
where $\varepsilon = \min(R/2, \rho/4)$, and $|C_N|$ is the volume of the necessary condition manifold.
\end{theorem}
\begin{proof}
Using Lemmas \ref{lem:R_complete} and \ref{lem:rho_complete}, we have established that if we have the conditions $\varepsilon \leq R/2$ and $\varepsilon \leq \rho/4$ where $\rho \geq 2h_m$, it is sufficient to have at least one sample in each ball $B_{\varepsilon}(p_i), i=0,...,n-1$ in order to produce a graph that achieves two properties:

\begin{enumerate}
  \item The graph contains at least one path from $p_\text{start}$ to $p_\text{goal}$ which passes entirely through $C_N$,
  \item The region covered by circles of radius $\rho$, centered at each vertex along one of the paths from $p_\text{start}$ to $p_\text{goal}$ will cover the entirety of $\{\varsigma_1,...,\varsigma_M\}$.
\end{enumerate}
  
Therefore, we choose $\varepsilon = \min(R/2, \rho/4)$, and then the probability that one of the balls $B_{\varepsilon}(p_i)$ will fail to be sampled is no greater than the expression given by equation \ref{eqn:pg_complete}.
\end{proof}

\subsection{Overall Completeness}

The success of the Mode Sampling stage requires the Possibility Exploration stage to succeed in finding a viable candidate path. Similarly, the success of the Multi-modal PRM stage requires the Mode Sampling stage to succeed in finding a set of modes that can reach from the start to the goal. Here we prove that the combination of these dependent processes is probabilistically complete given that the individual processes are each probabilistically complete.

\begin{lemma}
\label{lem:af_bf}
Consider the randomized processes $A$ and $B$. Suppose $B$ depends on $A$ such that $B$ can only succeed after $A$ has succeeded. Given $\Pr[\bar{A}] \leq a_F$ and $\Pr[\bar{B}|A] \leq b_F$, then the probability of both processes failing, $\Pr[\bar{A} \cup \bar{B}]$, is no greater than $a_F + b_F$.
\end{lemma}

\begin{proof}
Define the probability of process $A$ succeeding as $\Pr[A]$ and the probability of it failing as $\Pr[\bar{A}]$. If process $B$ cannot succeed unless process $A$ succeeds, then we know $\Pr[\bar{B}|\bar{A}] \equiv 1.0$ and $\Pr[A|B] \equiv 1.0$. If we are also given $\Pr[\bar{A}] \leq a_F$ and $\Pr[\bar{B}|A] \leq b_F$, we can derive the following:
\begin{equation}
\begin{split}
\Pr[\bar{A} \cup \bar{B}] & = \Pr[\bar{A}] + \left(1 - \Pr[\bar{A}]\right)\Pr[\bar{B}|A] \\
                          & \leq a_F + b_F
\end{split}
\end{equation}\end{proof}


%

%


\begin{theorem}
The probability of the overall process of the \mbox{w-RPG} failing to find a solution will asymptotically converge to zero as the number of samples used for each stage in the process goes to infinity.
\end{theorem}

\begin{proof}
Consider the Possibility Exploration stage to be process $A$ and the Mode Sampling stage to be process $B$. From Theorems \ref{thm:mode_complete} and \ref{thm:pg_complete}, we get the following expressions:
\begin{equation}
\begin{split}
a_F & \leq \frac{L}{\varepsilon}{\left(1 - \frac{\pi^3 \varepsilon^6}{6|C_N|}\right)}^{N_P}\\
b_F & \leq M \left(1 - \frac{r_m^2 {\Delta \theta}_m}{2|\mathcal{F}_\sigma|} \right)^{N_\sigma}
\end{split}
\end{equation}

We can use the inequality $(1-x) \leq e^{-x}$, for $x \geq 0$ to change these expressions to:
\begin{equation}
\label{eqn:af_bf_desc}
\begin{split}
a_F & \leq \frac{L}{\varepsilon}\exp\left(-\frac{\pi^3 \varepsilon^6}{6|C_N|}N_P\right)\\
b_F & \leq M\exp\left(-\frac{r_m^2 {\Delta \theta}_m}{2|\mathcal{F}_\sigma|}N_\sigma\right)
\end{split}
\end{equation}

Observing that $\alpha_1\exp\left(-\beta_1\right) + \alpha_2\exp\left(-\beta_2\right) \leq \alpha\exp(-\beta)$ where $\alpha = \alpha_1 + \alpha_2$ and $\beta = \min(\beta_1, \beta_2)$ and combining Lemma \ref{lem:af_bf} with the expressions in equation \ref{eqn:af_bf_desc}, we can get

\begin{equation}
\begin{split}
\Pr[\bar{A} \cup \bar{B}] \leq \left(\frac{L}{\varepsilon}+M\right)\exp\left(-\beta\right)\\
\beta = \min\left(\frac{\pi^3 \varepsilon^6}{6|C_N|}N_P, \frac{r_m^2 {\Delta \theta}_m}{2|\mathcal{F}_\sigma|}N_\sigma\right)
\end{split}
\end{equation}

Therefore, as both $N_P$ and $N_\sigma$ go to infinity, the probability of their combined process failing asymptotically approaches zero, making the combined process probabilistically complete.

This argument can be repeated recursively by viewing the combined process of Possibility Exploration and Mode Sampling as a single process upon which the Multi-modal PRM stage depends. Since Multi-modal PRM is known to be probabilistically complete, adding it as a dependent process onto another probabilistically complete process allows the overall process to still be probabilistically complete.
\end{proof}

\section{Analysis}

The expressions which have been derived to prove the probabilistic completeness of the \mbox{w-RPG} also reveal that the multi-stage procedure can offer a better rate of convergence for success than a single-stage procedure would. The parameter $\rho$ is used to restrict the region from which foot placements are sampled during the Mode Sampling stage. This focuses the mode sampling around the candidate route found in the Possibility Exploration stage, ensuring that the samples are conducive toward finding a solution as illustrated in Fig. \ref{fig:rho_effect}. As $\rho$ approaches infinity, the behavior is analogous to eliminating the Possibility Exploration stage altogether and instead merely sampling foot placements uniformly throughout the environment. In this section, we show how the earlier proofs predict an improvement in convergence. We also show simulation results which empirically reinforce this prediction.

\subsection{Theoretical Analysis}


Recall that $\mathcal{F}_\sigma$ is the union of circles with radius $\rho$, centered around the $\ceil{L/\varepsilon}-1$ vertices of the projected route from the Possibility Exploration stage. This gives us an upper bound on the area covered by $\mathcal{F}_\sigma$:
\begin{displaymath}
|\mathcal{F}_\sigma | \leq \frac{L}{\varepsilon}\pi \rho^2 \leq 4L\pi \rho
\end{displaymath}
Substituting this into equation \ref{eqn:mode_complete} for $\mathcal{F}_\sigma$, we get an upper bound on the likelihood of failure for Mode Sampling in terms of $\rho$:
\begin{equation}
\label{eqn:conv_ms}
\Pr[\text{Mode Sampling Failure}] \leq M\left(1 - \frac{r_m^2 {\Delta \theta}_m}{8L\pi \rho}\right)^{N_\sigma}
\end{equation}
which implies that minimizing $\rho$ will maximize the rate of convergence for the Mode Sampling stage.

However, there are limits to how small $\rho$ can be shrunk for the formula to hold. In particular, the proof for Theorem \ref{thm:mode_complete} depends on the assumption that $\mathcal{F}_\sigma$ covers the cylinders associated with the parameters $r_i$ and ${\Delta \theta}_i$. If $\mathcal{F}_\sigma$ is shrunk to no longer cover those cylinders, then the formula will not hold and we can no longer guarantee probabilistic completeness or asymptotic convergence. To ensure the formulae hold, the proof for Theorem \ref{thm:pg_complete} suggests $\rho \geq 2\mathcal{R}_\text{max}$ as a lower bound.

It is worth noticing that the formula also predicts the existence of pathological cases which cannot be reliably solved by the w-RPG. Specifically, if $r_m$ or ${\Delta \theta}_m$ have a value close to zero, then it implies that the solution requires a sample from a manifold with nearly zero volume in SE(2). The probability of randomly sampling a point on such a manifold is close to zero, so we could not expect this approach to reliably work, much like the well-known ``narrow passage problem'' \cite{sun2005narrow}. There would need to be some additional information provided to the planner that would allow it to find samples on that smaller manifold. For example, \citet{chestnutt2007adaptive} used evaluations of the terrain data to adjust infeasible footstep locations.

\subsection{Test Results}

\begin{figure}
  \centering
  \begin{subfigure}[t]{0.48\linewidth}
    \captionsetup{justification=centering}
    \centering
    \includegraphics[width=\linewidth]{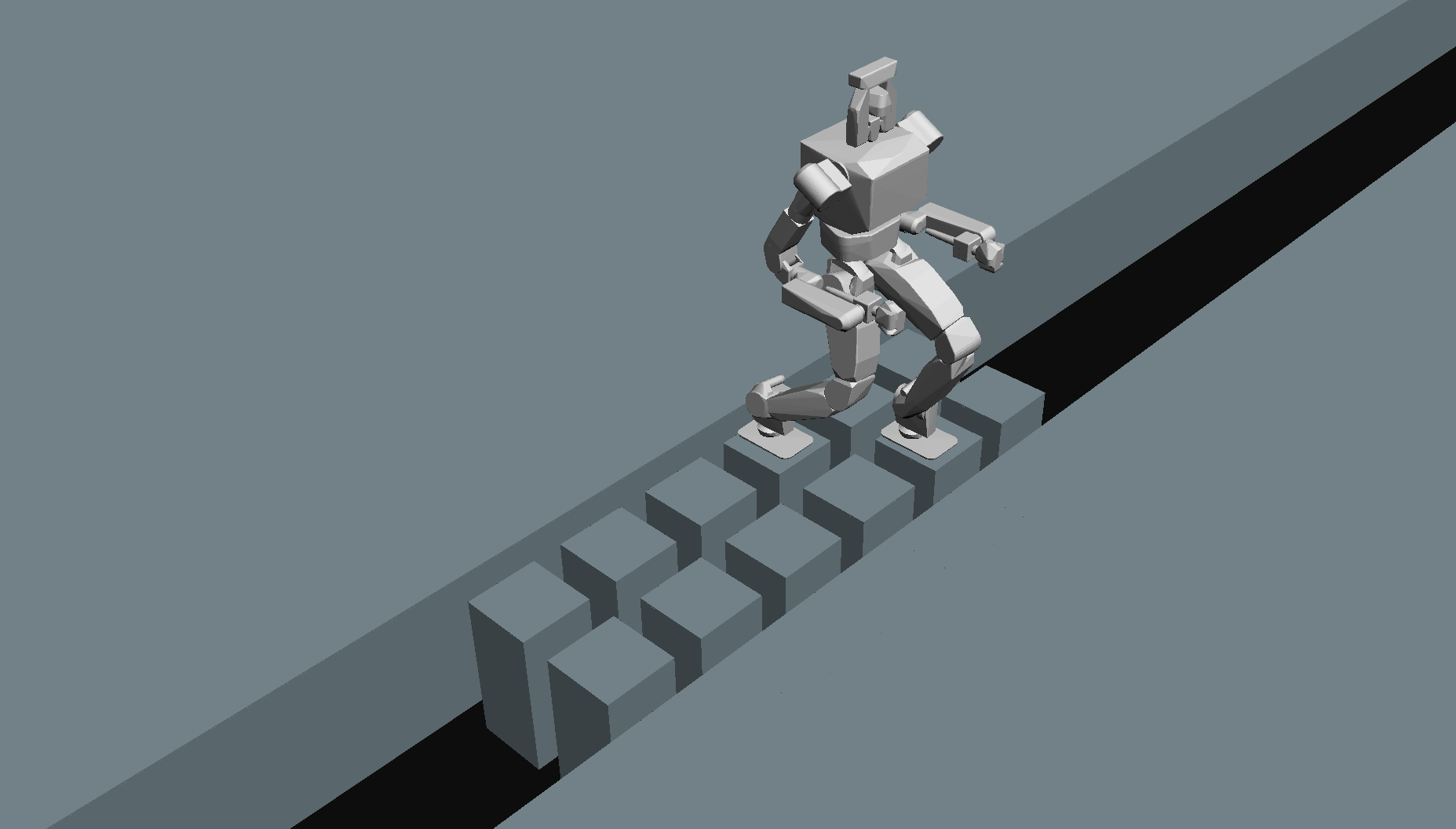}
    \caption{Stepping Stones Scenario}
    \label{fig:stepstones}
  \end{subfigure}
  \hfil
  \begin{subfigure}[t]{0.48\linewidth}
    \captionsetup{justification=centering}
    \centering
    \includegraphics[width=\linewidth]{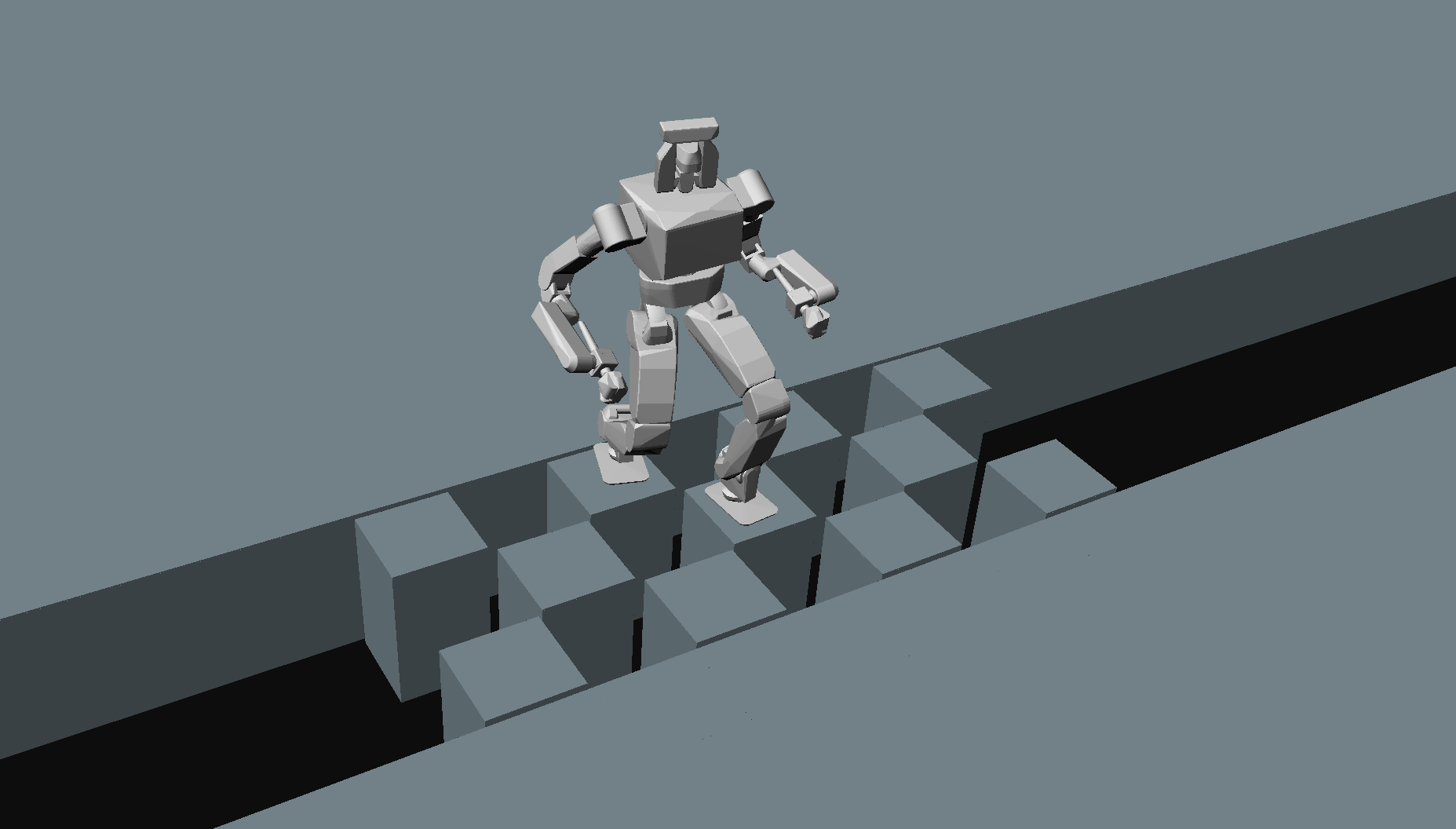}
    \caption{Checkers Scenario}
    \label{fig:checkers}
  \end{subfigure}
  
  \begin{subfigure}[t]{0.48\linewidth}
    \captionsetup{justification=centering}
    \centering
    \includegraphics[width=\linewidth]{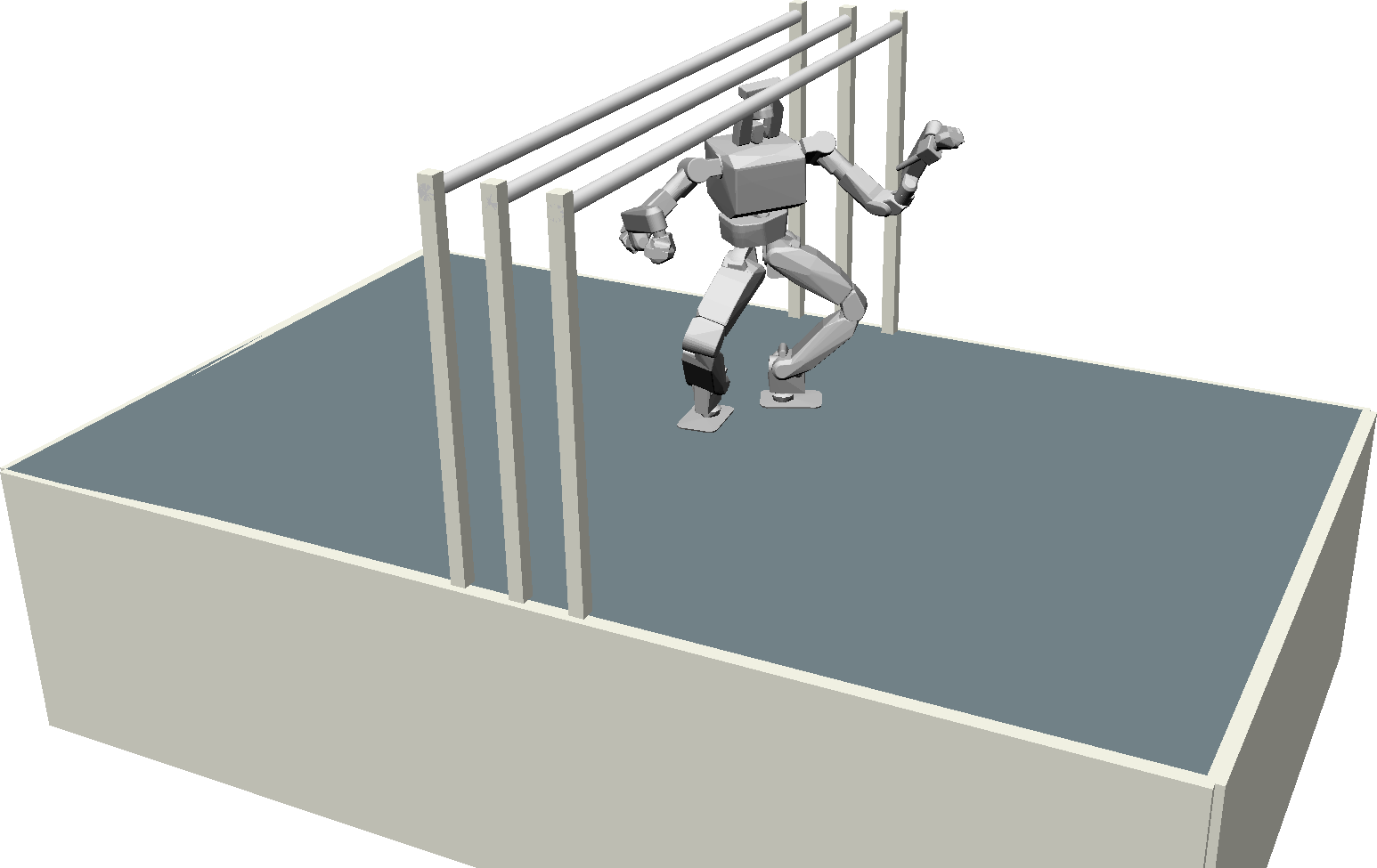}
    \caption{Pass Under Scenario}
    \label{fig:passunder}
  \end{subfigure}
  \caption{Three scenarios used for simulation tests. In (a) and (b), the robot must get across a gap by taking advantage of narrow stepping stones. In (c), the robot must pass underneath a sequence of bars.}
  \label{fig:scenarios}
\end{figure}

\begin{figure}
  \centering
  \includegraphics[width=0.9\linewidth]{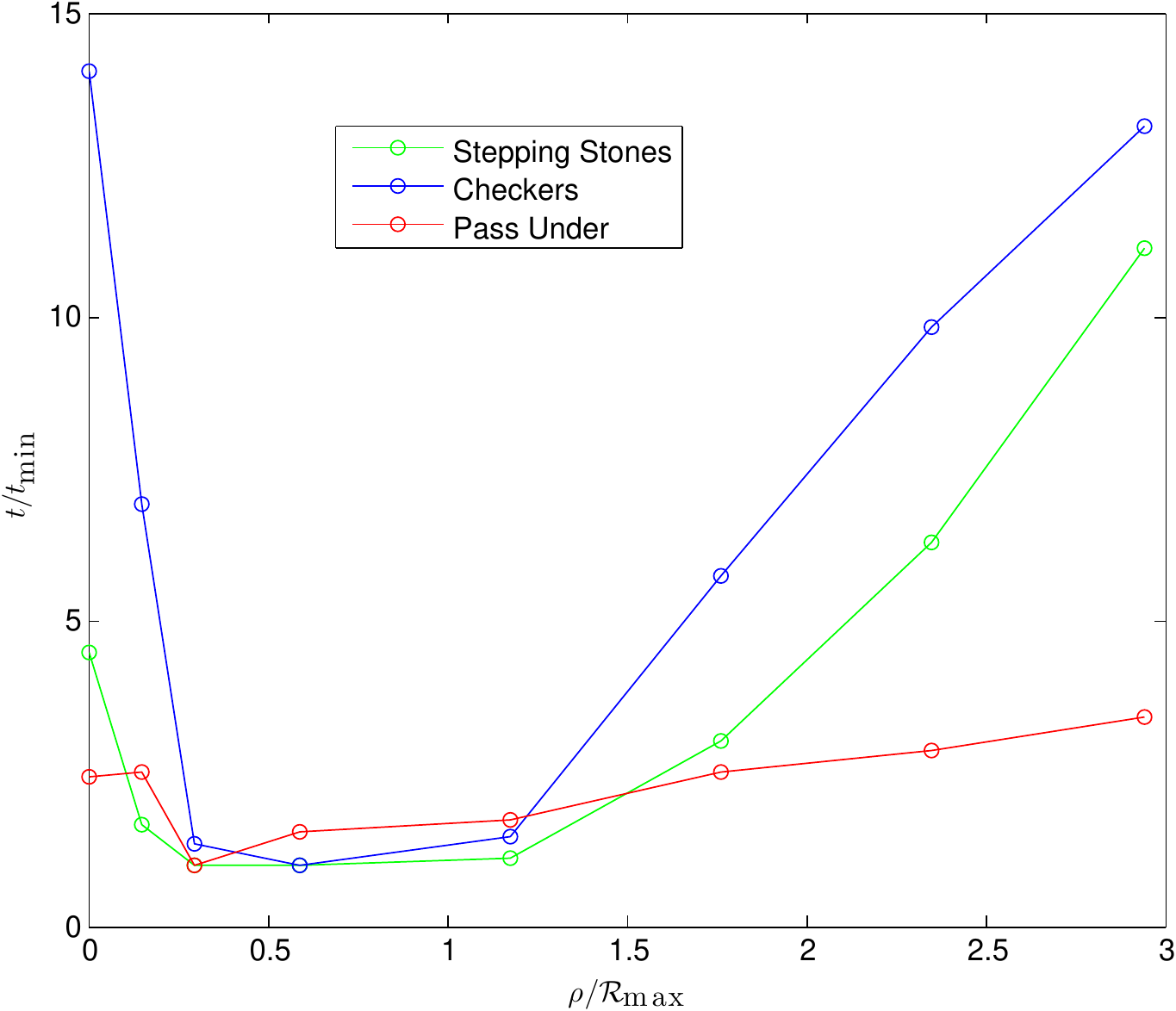}
  \caption{Average performance results from three scenarios, illustrating the relationship between $\rho$ and the rate of convergence. The $y$-axis shows the average time for each scenario, scaled by the data point with the smallest value. The $x$-axis shows how $\rho$ was varied, scaled by $\mathcal{R}_\text{max}$, the furthest distance that the robot is able to step.}
  \label{fig:data}
\end{figure}

To empirically test the effects of $\rho$, we constructed three simple scenarios and ran simulated tests while varying the value of $\rho$. Screenshots of the scenarios can be seen in Fig. \ref{fig:scenarios}. A plot of the results is shown in Fig. \ref{fig:data}. For the ``Stepping Stones'' and ``Checkers'' scenarios, the robot needs to find a sequence of foot placements that can get it across a wide gap. In such cases, Mode Sampling is the primary bottleneck, and equation \ref{eqn:conv_ms} plays the dominant role. This gives us performance results which reflect the theoretical predictions of the lower bound.

In contrast, the ``Pass Under'' scenario requires the robot to pass underneath a sequence of three bars. The floor is clear of holes or obstructions, leaving it wide open for the robot to place its feet anywhere. Furthermore, the overall floor space of the environment is relatively small. These factors result in a scenario where Mode Sampling is a less demanding stage. Instead, the physical obstacle of the overhanging bars results in a narrow passage, making $R$ the deciding variable for $\varepsilon$ in equation \ref{eqn:pg_complete}. Smaller values for $\rho$ may still offer some marginal performance improvements, but it does not appear to be exponential as it is for the other two scenarios.

While the theoretical analysis proposes a value of $\rho = 2\mathcal{R}_\text{max}$ to optimize performance while guaranteeing probabilistic completeness, the empirical data suggests that a value in the range $\frac{1}{2}\mathcal{R}_\text{max} \leq \rho \leq \mathcal{R}_\text{max}$ might be best for performance in practice. A potential strategy could be to schedule the value of $\rho$ so that it begins with a high-performance value and then grows up to the theoretical lower bound over time.

%
%
%
%
%
%

\section{Conclusion} 
\label{sec:conclusion}

In this paper we have provided a proof for the probabilistic completeness of w-RPG which represents a worst-case performance of the Randomized Possibility Graph (RPG) algorithm applied to bipedal locomotion planning. We also demonstrated, both theoretically and empirically, how the Possibility Exploration stage of the algorithm allows the overall process to converge more quickly by focusing the effort of the lower-level stages into regions of the environment that are relevant for finding a solution.

There are two crucial limitations to the existing implementation of the RPG algorithm. The first limitation is that it currently only samples foot placements from SE(2), making it unable to handle uneven terrain. We aim to address this in later work by leveraging reachable space representations similar to \citet{tonneau2015reachability}. That should allow the algorithm to handle arbitrary obstacles \textit{and} arbitrary terrain. Second, this implementation currently only extends to quasi-static motion. The work of \citet{dellin2012framework} may provide a natural complementary low-level planner for the RPG, allowing it to plan highly dynamic motions in addition to quasi-static.

%

\bibliographystyle{plainnat}
\bibliography{references}

\end{document}